\documentclass[12pt,reqno]{amsart}

\usepackage{latexsym}
\usepackage{amsmath}
\usepackage{amsrefs}
\usepackage{amscd}
\usepackage{amssymb}
\usepackage{mathrsfs}
\usepackage{mathtools}
\usepackage{comment}
\usepackage{color,framed}
\usepackage{enumerate}
\usepackage{soul}
\usepackage{stmaryrd}
\usepackage[ruled]{algorithm2e} 
\usepackage{braket}

\usepackage{graphicx}
\usepackage{subfigure}
\usepackage{url}

\newtheorem{theorem}{Theorem}[section]

\newtheorem{proposition}[theorem]{Proposition}

\numberwithin{equation}{section}

\begin{document}

\title[Stochastic metamorphosis with template uncertainties]{Stochastic metamorphosis with template uncertainties}

\author[A. Arnaudon]{Alexis Arnaudon}
\author[D. Holm]{Darryl D Holm}
\author[S. Sommer]{Stefan Sommer}

\address{AA, DH: Department of Mathematics, Imperial College, London SW7 2AZ, UK}
\address{SS: Department of Computer Science (DIKU), University of Copenhagen,
  DK-2100 Copenhagen E, Denmark}

\subjclass[2010]{60G99,70H99,65C30}
\maketitle

\begin{abstract}
     In this paper, we investigate two stochastic perturbations of the metamorphosis equations of image analysis, in the geometrical context of the Euler-Poincar\'e theory. 
    In the metamorphosis of images, the Lie group of diffeomorphisms deforms a template image that is undergoing its own internal dynamics as it deforms. 
    This type of deformation allows more freedom for image matching and has analogies with complex fluids when the template properties are regarded as  order parameters (coset spaces of broken symmetries). 
    The first stochastic perturbation we consider corresponds to uncertainty due to random errors in the reconstruction of the deformation map from its vector field. We also consider a second stochastic perturbation, which compounds the uncertainty in of the deformation map with the uncertainty in the reconstruction of the template position from its velocity field. 
    We apply this general geometric theory to several classical examples, including landmarks, images, and closed curves, and we discuss its use for functional data analysis. 
\end{abstract}

\setcounter{tocdepth}{2}


\section{Introduction}
Variability in shapes can be modelled using flows of the group $G$ of diffeomorphic deformations of the ambient domain $\Omega$ in which the shape is embedded. This is the basis of the large deformation diffeomorphic metric mapping  (LDDMM) framework, see \cite{trouve_infinite_1995,christensen_deformable_1996,dupuis_variational_1998,beg2005computing}.
In the LDDMM approach, the shape of an embedded template image $\eta\in N$ in the manifold of embedded shapes ${\rm Emb}(N,\Omega)$ changes via the action $g_t.\eta$ of time-dependent diffeomorphisms $g_t\in G$ on $\eta\in N$, through the action of $g_t$ on the domain $\Omega$. The metamorphosis extension \cite{holm2009euler,miller2001group,trouve2005local,trouve2005metamorphoses} of LDDMM introduces a further time-dependent variation $\eta_t$ of the template to model the combined dynamics $g_t.\eta_t$. 

In this paper, we combine the geometrical metamorphosis framework  of \cite{holm2009euler} with 
recent developments in stochastically perturbed Euler-Poincar\'e dynamics in
fluid dynamics and shape analysis
\cite{holm2015variational,arnaudon2017geometric,arnaudon2016noise}, to model
evolutions of both shape and template under stochastic perturbations.
The resulting framework allows modelling of random evolutions of shape and
template simultaneously. A potential application of such an evolution is in modelling  the progression of disease using computational anatomy, in which the model would address the analysis of disease progression in both the population average and in the individual. From longitudinal image data, mean
evolutions over the population can be inferred. While average template
evolutions can be modelled deterministically, models for the dynamics of each individual subject that include stochastic uncertainty are arguably more realistic
than models supporting only smooth deterministic trajectories. The stochastic metamorphosis model
includes such non-smooth and non-deterministic variations by incorporating stochastic perturbations in shape and template
simultaneously. We detail this application and outline further areas of applications where
similar generative models of data appear, particularly in the combined modelling
of phase and amplitude variation in functional data analysis.

\begin{figure}[!th]
\centering
  \vspace{.4cm}
\mbox{
\def\svgwidth{0.6\columnwidth}
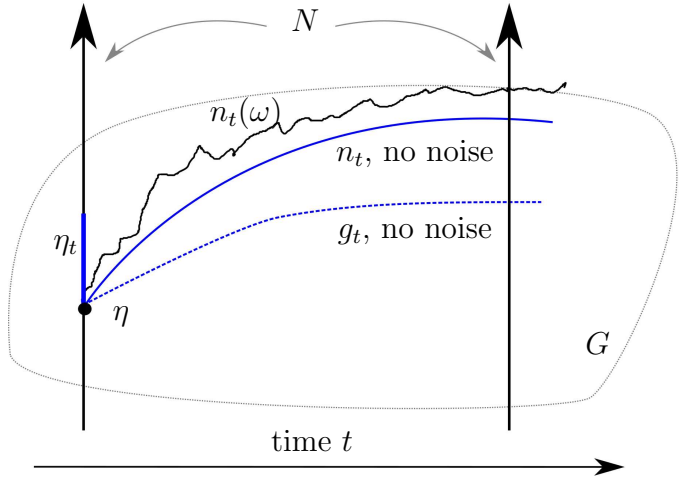
}
\caption{Sketch of the deformation group $G$, the shape space $N$ (vertical
  arrows), evolution of the deformation variable $g_t$, template variable
  $\eta_t$, and shape variable
  $n_t$ without noise ($W_t=0$, blue), and shape variable $n_t$ with noise
  $\omega$ (black).
  The shape space is illustrated as being linear (e.g. landmarks, images).
  However, the framework applies to general non-linear shape spaces (e.g. curves,
  tensor fields).
}
\label{fig:sto_metamorphosis_sketch}
\end{figure}

\subsection{Background}

The LDDMM framework models the change of a shape $\eta\in N$ by the action of time dependent flows of diffeomorphisms $g_t\in G$ on the embedding space $\Omega$. One lifts the shape trajectory 
to a time-dependent curve $g_t$ on the diffeomorphisms by setting $n_t=g_t.\eta\in N$. 
For a right invariant metric on the tangent space of a subgroup $G$ of the diffeomorphism group $\mathrm{Diff}(N)$, $N$ being the shape
space, an energy can be defined as
$E(g_t)=\int_0^1\|\partial_t g_t\|_{g_t}^2{\rm d}t$. Combined with a data
attachment term, this approach allows matching of
shapes and image registration \cite{beg2005computing}. The invariance of $E(g_t)$ under the right action of $G$ implies that the metric descends to a metric structure on the data space $N$ itself. The action of $g_t$ differs between data types, but otherwise, the framework is formally equivalent for different classes of shapes. 
The use of the flows $g_t$ to model the shape variability is fundamental and the right trivialization $v_t:=\partial_t g_t\circ g_t^{-1}$ gives an Eulerian
interpretation of the metric. The right invariance of the metric enables Euler-Poincar\'e reduction of the dynamics to the Lie algebra of $G$ to be performed, and the critical paths for $E$ appear from the reduced dynamics.

Metamorphosis extends the LDDMM setting by letting the template vary in time
as well as the deformation, thereby resulting in the flow $g_t.\eta_t$, in which $\eta_0=\eta$ is the original template. The metamorphosis energy
is encoded into a Lagrangian depending on both the $G$ and $N$ variability, again assuming
invariance of the energy to the group action on both $G$ and $N$. 
A particular example of metamorphosis dynamics arises in image analysis,
where the image $I_t$ changes both by deformation via the right action
$g_t.I_t=I_t\circ g_t^{-1}$ and via a pointwise change $\partial_t I_t(x)$ for
each pixel/voxel $x$. 

In this formulation of metamorphosis dynamics, an analogy with the flows of complex fluids arises. In complex fluids, a diffeomorphic flow carries an order parameter, defined as a coset space for a broken symmetry of homogeneous fluids, on which the diffeomorphisms act. The order parameter moves with the fluid, but it can also have its own internal dynamics, which in turn is coupled to the fluid motion \cite{holm2002euler, gay2009geometric}. 
A similar combined dynamics of shape and template also appears in the Fshape framework
\cite{charlier_fshape_2017}.

In \cite{arnaudon2017geometric,arnaudon2016noise}, a stochastic model of shape
evolution was introduced that preserves the Euler-Poincar\'e theory of the
deterministic LDDMM framework. The model is based on the stochastic fluid dynamics
model \cite{holm2015variational} where right-invariant noise is introduced
to perturb the reconstruction equation that evolves the flow from the reduced
dynamics. In deterministic LDDMM, the reconstruction equation specifies the evolution
of the group element by $\partial_tg_t=v_t\circ g_t$ generated by the reduced Eulerian
velocity vector field $v_t$. Stochasticity is
introduced as a perturbation to the reconstruction equation, by introducing the stochastic
time differential 
\begin{equation}
  {\rm d}g_t \,g_t^{-1}=v_t{\rm d}t+\sum_{l=1}^N\sigma_l\circ {\rm d}W_t^l
\label{circ-diff}
\,.\end{equation}
Here $W_t^l$ are standard Wiener processes and $\sigma_l$ are vector fields on the data domain which characterize the spatial correlation of the noise. As it turns out, the noise in Stratonovich form is denoted conventionally with the same symbol $(\circ)$ that denotes composition of maps. This coincidence should not cause any confusion. However, just to be sure, we will write composition of maps as concatenation whenever the two meanings appear in the same equation, as in \eqref{circ-diff}.
The perturbation of the reduced variable implies that the noise is right-invariant
and in a certain sense compatible with the right-invariant LDDMM metric. This approach
preserves many of the geometric structures of the deterministic framework.
Importantly, the descent of the stochastic model to particular data types is 
similar to the way the metric descends in the deterministic LDDMM framework.

A stochastic metamorphosis extension of the stochastic Euler-Poincar\'e
framework was introduced in \cite{holm2017stochastic2}. The stochastic perturbations there were also introduced in the reduced variable influencing the deformation flow from the reconstruction equation. The template evolution
$\eta_t$ is still deterministic. The aim of the present paper is to extend this
model to include noise in the template evolution $\eta_t$ as well. We will make 
this extension on the reduced template velocity $g_t\partial_t\eta_t$
similarly to the perturbation of the group variable. This procedure results in simultaneous stochastic
perturbations of the flow equations for both $g_t$ and $\eta_t$. 

\subsection{Paper outline}
After a brief survey of the deterministic metamorphosis framework in section \eqref{det-review}, we formally derive 
the stochastic model in section \eqref{formal-sto}. 
We then show in section \eqref{HP-derivation} how to derive these equations in the Hamilton-Pontryagin formulation, where the noise appears as a stochastic constraint in the variational principle. 
We end the theoretical section by deriving the corresponding Hamiltonian stochastic equations in \eqref{ham-form} to then move to some classical examples of image analysis and computational anatomy in section \eqref{applications}, including landmarks and images.
The inclusion of two types of stochastic variations links the framework to combined random phase and
amplitude variations in functional data analysis. We provide perspectives of the
method to future applications in functional data analysis and computational
anatomy in section \eqref{perspective}.

\section{General stochastic metamorphosis}

In this section, we introduce the stochastic deformation of metamorphosis, but first, we recall the basis of this theory, in the context of reduction by symmetry. 
We will only review what will be needed for our exposition, and we refer to \cite{holm2009euler} and \cite{holm2017stochastic2} for more extensive treatments. 

\subsection{Deterministic metamorphosis}\label{det-review}

The theory of metamorphosis begins with a template $N$, considered here as a manifold (landmarks, images, etc\dots) upon which a group of diffeomorphism $G= \mathrm{Diff}(N)$ acts. 
The parameter space of this theory is $G\times N$, with curves $(g_t,\eta_t)\in G\times N$, where $g_t$ is the deformation curve and $\eta_t$ is the template curve. 
The image curve will be denoted $n_t = g_t.\eta_t\in N$, where the dot represents the group action. 
This curve is the total motion of the template, or image $N$, under both the deformation and its own dynamics. 
For standard LDDMM, the motion of the image is only $n_t= g_t.\eta$, for a fixed reference template $\eta$. 
This combined action thus allows more freedom in the matching procedure, while remaining compatible with the theory of reduction by symmetries, which we now describe.  
We first define the two reduced velocity fields
\begin{align}\label{2red-vel-VFs}
  u_t := \dot g_t g_t^{-1}
  \,,\quad\hbox{and}\quad
  \nu_t : = g_t \dot \eta_t\,.
\end{align}
The first is the reduced deformation velocity and the second is the reduced template velocity. 
We then assume that the original Lagrangian of this theory is invariant under the group action of $G$, so that we may write the reduced Lagrangian in terms of the reduced velocity fields and the image position $n_t$, i.e.,
\begin{align}
      L(g_t, \dot g_t, \eta_t, \dot \eta_t) = l(u_t, n_t, \nu_t)\, . 
\end{align}
Because the reduced Lagrangian still depends on the template variable $n_t$, reduction by the action of the diffeomorphisms will result in a semi-direct product structure, where the template is an advected quantity, in the language of fluid dynamics. 

We next compute the variations of the three variables in the reduced Lagrangian, upon introducing the notation $\xi_t = \delta g_t g_t ^{-1}$ and $\omega = g\delta \eta$, where $\delta g$ and $\delta \eta$ are free variations, to obtain  
\begin{align}
\begin{split}
    \delta u &= \dot \xi_t - [u_t, \xi_t] \,,\\
    \delta n &= \omega_t + \xi_t \eta_t \,,\\
    \delta \nu &= \dot \omega_t + \xi_t \nu_t - u_t \omega_t\, . 
\end{split}    
\end{align}
In these formulas, we need to specify what we mean by the multiplication, as the vector fields live in different spaces. 
In fact, $u_t,\xi_t\in \mathfrak g= \mathfrak X(N)$ are vector fields; so the Lie bracket is the natural operation. 
Recall that $\eta_t\in N$, thus $\xi_t \eta_t$ corresponds to the tangent map of the action of $G$ on the manifold $N$, and similarly for $\nu\in TN$, where the action is on the tangent space of $N$. 
We do not need these actions explicitly now, but we will need their `adjoint action' in the following sense:
\begin{align}
  \langle n^*\diamond m, u\rangle_\mathfrak{g} = - \langle n^*, un\rangle_{N}
  \,,\\
  \langle u\star \nu^*, \nu \rangle_{TN} =  \langle \nu^*, u\nu \rangle_{TN}\,,
  \label{ad-star-def}
\end{align}
where $N\in N$, $n^*\in N^*$, $\nu \in TN$, $\nu^*\in T^*N$ and $u\in \mathfrak g$. 
The first equality defines the diamond operation $(\diamond)$, which will serve as a force term to capture the coupling between the advected quantity $n$ and the main dynamics of the diffeomorphism group. The second equality defines the star operation $(\star)$, which is the adjoint of the action of $u$ on $TN$. That is, it defines the action of $u$ on $T^*N$. 

Applying the variational calculus to the action $S= \int l {\rm d}t$, we obtain the Euler-Poincar\'e formulation of the metamorphosis equation in the form
\begin{align}
  \begin{split}
    \frac{{\rm d}}{{\rm d}t}\frac{\delta l}{\delta u} &+ \mathrm{ad}^*_{ u_t} \frac{\delta l}{\delta u} + \frac{\delta l}{\delta n}\diamond n +  \frac{\delta l}{\delta \nu}\diamond \nu = 0 \,,\\
    \frac{{\rm d}}{{\rm d}t} \frac{\delta l}{\delta \nu} &+ u_t \star \frac{\delta l}{\delta \nu} - \frac{\delta l}{\delta n}= 0\, , 
  \end{split}
    \label{EP-metamorpho}
\end{align}
together with the reconstruction equation
\begin{align}
    \dot n &= u_t n_t + \nu_t\,  . 
    \label{reconstruction-relation}
\end{align}
We refer to \cite{holm2009euler, holm2017stochastic2} for the details of this derivation, which we will do in the context of Hamilton-Pontryagin with noise in the next section. 

From here, a choice of Lagrangian and data $N$ will reduce the system to particular cases, some of which we discuss in the applications section \ref{applications}. 

\subsection{Formal derivation of the stochastic equations}\label{formal-sto}

We will first derive the equation informally, using `stochastic variations', then show a more straightforward derivation using the Hamilton-Pontryagin principle. 
The second derivation also has the advantage of revealing the effects of the noise more transparently. 

In order to introduce a noise compatible with the Euler-Poincar\'e equation, we need to perturb the theory at its core, which is in this case the definition of the reduced velocities in \eqref{2red-vel-VFs}. 
Indeed, the variations were computed from these definitions, and the deterministic Euler-Poincar\'e equation emerged. 
Although a single relation is used in the Euler-Poincar\'e equation \eqref{reconstruction-relation}, we will split it into two parts, and perturb them with two different noise components as follows,
\begin{align}
  \begin{split}
    {\rm d}g_t g_t^{-1} &= u_t(x) {\rm d}t + \sum_{l=1}^{K^u} \sigma^u_l(x) \circ {\rm d}W_t^l =: {\rm d} u_t(x)
  \,,\\
  g_t {\rm d}\eta &= \nu_t {\rm d}t + \sum_{k=1}^{K^\nu} \sigma^\nu_k \circ {\rm d}W_t^k =: {\rm d} \nu_t\, . 
  \end{split}
  \label{sto-reconstruction}
\end{align}
In a slight abuse of notation, ${\rm d} u_t(x)$ and ${\rm d} \nu_t$ are written as stochastic processes. 
Here $\sigma_l^u:N\to \mathfrak g$ are a set of $K^u$ vector fields on the domain $\Omega$, and $\sigma^\nu_l\in TN$ are another set of $K^\nu$ tangent vectors on the template. 
We also denote by $W_t^l$ or $W_t^k$ the $K^u+ N^\nu$ independent standard Weiner processes. 
In addition, we denote by $x_0\in \Omega$ the Lagrangian labels upon which $g_t$ acts, so that the first equation can be written equivalently as 
\begin{align*}
  {\rm d}g_t &= u_t(g_tx_0) {\rm d}t 
  + \sum_{l=1}^{K^u} \sigma^u_l(g_t x_0) \circ {\rm d}W_t^l\, . 
\end{align*}
The second equation for $\eta$ in \eqref{sto-reconstruction} does not have any $x_0$ dependence, as it is an equation for the template itself. Thus, $\sigma_k^\nu$ are not functions of $N$; rather, they are tangent vectors to $N$. 

With the notation for ${\rm d}u_t$ and ${\rm d}\nu_t$ in \eqref{sto-reconstruction}, we have the complete reconstruction relation for the stochastic image template $n_t$
\begin{align}
    {\rm d}n_t = {\rm d}u_t \,n_t + {\rm d}\nu_t\, . 
    \label{dn-reconstruction}
\end{align}
Because $n_t\in N$, the concatination ${\rm d}u_t\, n_t$ means the composition ${\rm d} u_t(n_t)$.  
In  \eqref{dn-reconstruction}, the noise in the $u_t$ vector field was introduced in \cite{holm2017stochastic2}, based on the stochastic fluid dynamics model of \cite{holm2015variational}, whereas the noise in the $\nu_t$ field is new.
The first noise term in  \eqref{dn-reconstruction} corresponds to random errors in the reconstruction of the diffeomorphism path from its velocity field, while the second one represents random errors for the reconstruction of the template position from its velocity field. 
In stochastic metamorphosis, the two noise terms will affect the dynamical equations differently.

From these stochastic perturbations of the reconstruction relation, we can formally compute the variations and obtain
\begin{align}
\begin{split}
    \delta u &= {\rm d}\xi_t + [\xi, {\rm d}u_t]\,,\\
    \delta \nu &= {\rm d}\omega + \xi {\rm d} \nu_t - {\rm d}u_t \omega\, .
\end{split}
\label{stoch-vars}
\end{align}
These are convenient expressions, but they introduce the variations as stochastic processes; so they should not be taken at face value without further analysis. We will see in the next section how to re-derive these equations without introducing stochastic variations, by using the Hamilton-Pontryagin principle. Because the results are identical for the two methods, we can proceed formally here by using these variations as we did in the deterministic variational principle to obtain the following stochastic reduced metamorphosis  equations in Euler-Poincar\'e form, 
\begin{align}
  \begin{split}
    {\rm d}\frac{\delta l}{\delta u} &+ \mathrm{ad}^*_{{\rm d}u_t} \frac{\delta l}{\delta u} + \frac{\delta l}{\delta n}\diamond n {\rm d}t +  \frac{\delta l}{\delta \nu}\diamond {\rm d}\nu_t = 0 \,,\\
    {\rm d}\frac{\delta l}{\delta \nu} &+ {\rm d}u_t \star \frac{\delta l}{\delta \nu} - \frac{\delta l}{\delta n} {\rm d}t= 0\, ,
  \end{split}
    \label{sto-metamorpho}
\end{align}
as well as equation \eqref{dn-reconstruction}, all to be compared with the deterministic case in equations \eqref{EP-metamorpho} and \eqref{reconstruction-relation}.

\subsection{Derivation using the Hamilton-Pontryagin principle}\label{HP-derivation}

We now show how to rederive the stochastic metamorphosis equations more transparently,  without introducing stochastic variations \eqref{stoch-vars}. 
For this purpose, we will use the stochastic Hamilton-Pontryagin approach and closely follow the exposition of \cite{holm2017stochastic2}.

The deterministic Hamilton-Pontryagin principle is a variational principle with the following constrained action
\begin{align}
  \begin{split}
    S(u_t,n_t, \dot n_t, \nu_t, g_t, \dot g_t) = \int_0^1 l(u_t,n_t,\nu_t){\rm d}t &+\int_0^1  \langle M_t, (\dot g_t g_t^{-1} - u_t) \rangle {\rm d}t \\
    &+ \int_0^1  \langle \sigma_t, (\dot n_t - \nu_t- u_tn_t) \rangle {\rm d}t \, , 
  \end{split}
\end{align}
where $M_t\in \mathfrak{X}^*(N)$ and $\sigma_t \in T^*N$ are generalised Lagrange multipliers to enforce the constraint of the reconstruction relations. 
Taking free variations for all the variables yields the deterministic reconstruction relation \eqref{reconstruction-relation} and the deterministic Euler-Poincar\'e equation \eqref{EP-metamorpho}. 
We refer to \cite{holm2017stochastic2} for more details of the derivation. The crucial point here is to allow free variations, by introducing constraints into the variational principle, and not in the variations as in the standard Euler-Poincar\'e reduction theory.  An alternative approach would be to use the Clebsch constrained variational method used for fluid dynamics in \cite{holm2015variational}.

In the present context, we enforce the stochastic reconstruction relations \eqref{sto-reconstruction} via the following stochastic Hamilton-Pontryagin principle
\begin{align}
  \begin{split}
  S(u_t,n_t, {\rm d} n_t, \nu_t, g_t, d g_t) = \int_0^1  l(u_t,n_t,\nu_t) &+ \int_0^1 \langle M_t, ({\rm d}g_t g_t^{-1} - {\rm d}u_t ) \rangle\\
  &+ \int_0^1\langle \sigma_t, ({\rm d} n_t - {\rm d}\nu_t- {\rm d}u_t n_t) \rangle  \, , 
  \end{split}
\end{align}
or, more explicitly, upon substituting for ${\rm d}u_t$ and ${\rm d}\nu_t$ from \eqref{sto-reconstruction}, we have
\begin{align}
  \begin{split}
  S(u_t,n_t, {\rm d} n_t, \nu_t, g_t, {\rm d}g_t) 
  &= \int_0^1  l(u_t,n_t,\nu_t) {\rm d}t \\
  &\quad + \int_0^1 \left \langle M_t, {\rm d}g_t g_t^{-1} - u_t {\rm d}t - \sum_{l=1}^{K^u} \sigma^u_l(x) \circ {\rm d}W_t^l  \right \rangle \\
  &\quad + \int_0^1 \left \langle \sigma_t, {\rm d}n_t - \nu_t {\rm d}t - \sum_{k=1}^{K^\nu} \sigma^\nu_k(x) \circ {\rm d}W_t^k \right \rangle \\
  &\quad - \int_0^1 \left \langle \sigma_t, \Big(  u_t {\rm d}t + \sum_{l=1}^{K^u} \sigma^u_l(x) \circ {\rm d}W_t^l\Big) n_t\right  \rangle \, . 
  \end{split}
     \label{sto-HP-principle}
\end{align}
\begin{proposition}
  The stochastic variational principle $\delta S= 0 $ with action \eqref{sto-HP-principle} yields the stochastic Euler-Poincar\'e equation \eqref{sto-metamorpho} with stochastic reconstruction relation \eqref{sto-reconstruction} and \eqref{dn-reconstruction}.
\end{proposition}
\begin{proof}
  The proof is a direct computation by taking free variations. We will show the key steps below. 
  First, the variations with respect to $M_t$ and $\sigma_t$ yield the reconstruction relations \eqref{sto-reconstruction} and \eqref{dn-reconstruction}. 
  Then, the variations with respect to $u_t,n_t$ and $\nu_t$  specify 
  \begin{align}
    \frac{\delta l}{\delta u_t} = M_t+\sigma_t\diamond n_t\, ,  \quad {\rm and} \quad \frac{\delta l}{\delta \nu_t} = \sigma_t\, . 
    \label{M-sig-def}
  \end{align}
  We also have, for the $n_t$ variations,
  \begin{align}
    \frac{\delta l}{\delta n_t} {\rm d}t = {\rm d} \sigma_t + u_t\star \sigma_t {\rm d}t + \sum_{l=1}^{K^u} \sigma^u_l(x)\star \sigma_t \circ {\rm d}W_t^l \, . 
    \label{sigma_eq}
  \end{align}
  Finally, for $\xi = \delta gg^{-1}$ vanishing at the endpoints, we have 
  \begin{align}
    \delta({\rm d}g_t g_t^{-1}) = {\rm d}\xi - \left [u_t {\rm d}t + \sum_{l=1}^{K^u} \sigma^u_l(x) \circ {\rm d}W_t^l, \xi\right] \, . 
  \end{align}
  From this computation, we have the last term in the calculus of variations  which reads
  \begin{align}
    {\rm d}M_t = -\,\mathrm{ad}^*_{u_t} M_t -\, \sum_{l=1}^{K^u} \mathrm{ad}^*_{\sigma^u_l(x)}M_t \circ {\rm d}W_t^l\, . 
    \label{M_eq}
  \end{align}
  Finally, substituting the values of $M_t$ and $\sigma_t$ of \eqref{M-sig-def} in equation \eqref{sigma_eq} and \eqref{M_eq} yields the stochastic metamorphosis equation \eqref{sto-metamorpho} after a few more manipulations (see Corollary 3 of \cite{holm2017stochastic2}). 
\end{proof}

\subsection{Hamiltonian formulation}\label{ham-form}

Provided that the Lagrangian is hyperregular, the stochastic metamorphosis equation \eqref{sto-metamorpho} can be written as a stochastic Hamiltonian equation with Hamiltonian obtained via the reduced Legendre transform, 
\begin{align}
  h(\mu,\sigma,n) = \langle \mu,u\rangle + \langle \sigma, \nu\rangle - l(u,\nu,n)\, , 
\end{align}
in which $\mu$ and $\sigma_t$ are the conjugate variables of $u_t$ and $\nu_t$, respectively. 
The noise is encoded into the stochastic potentials 
\begin{align}
  \Phi_l^u(\mu_t) = \langle \mu_t, \sigma_l^u\rangle_{\mathfrak g\times \mathfrak g^*}\, , \quad\mathrm{and} \quad   \Phi_k^\nu(\sigma_t) = \langle \sigma_t, \sigma^\nu_k\rangle_{TN\times T^*N}\, , 
  \label{sto-potential}
\end{align}
such that the stochastic equation of motion has a Hamiltonian drift term with $h$ and stochastic terms obtained via the same Hamiltonian structure, but with stochastic potentials. 
Notice that the two potentials have a different pairing, one on the Lie algebra of the diffeomorphism group, and the other on the tangent space of the template manifold. 
The Hamiltonian structure is given in \cite{holm2017stochastic2} and we will only display here the Hamiltonian equations
\begin{align}
  \begin{split}
    {\rm d}\mu_t &+ \mathrm{ad}^*_\frac{\delta h}{\delta \mu}u {\rm d}t  +  \sigma \diamond \frac{\delta h}{\delta \sigma } {\rm d}t +  \frac{\delta h}{\delta n}\diamond n {\rm d}t 
    \\&+\sum_l \mathrm{ad}^*_\frac{\delta \Phi_l^u}{\delta \mu}u\circ {\rm d} W_t^l+   \sum_l \sigma_t \diamond \frac{\delta \Phi^\nu_l}{\delta \sigma } \circ {\rm d}W_t^l  = 0 \,,\\
    {\rm d}\sigma_t &+ \frac{\delta h}{\delta \mu} \star \sigma_t  {\rm d}t - \frac{\delta h}{\delta n} + \sum_l \frac{\delta \Phi_l^u}{\delta \mu} \star \sigma_t \circ {\rm d}W_t^l = 0\, .
  \end{split}
    \label{sto-metamorpho-Ham}
\end{align}
In the examples in the next section, we will use this formulation to derive the stochastic equations of motion. Taking the Hamiltonian approach turns out to be more transparent than the Lagrangian description. 

\section{Applications}\label{applications}

Following \cite{holm2009euler}, we explicitly provide the stochastic metamorphosis equations for a few classical examples, including landmarks and images, and leave other applications such as closed planar curves, densities or tensor fields for later works. 

\subsection{Landmarks and peakons}

Consider the case when the template manifold $N$ is the space of $n$ landmarks $\mathbf q = (q_1, \ldots, q_n)\in \Omega^n$ with momenta $\mathbf p = (p_1, \ldots, p_n)\in T_\mathbf{q}\Omega^n \cong \Omega^n$. 
One needs to specify a Lagrangian for this system, and the simplest is
\begin{align}
  l(u,n,\nu) = \frac12 \|u\|_K^2+ \frac{\lambda^2}{2 } \sum_{i=1}^n |p_i|^2\, , 
\end{align}
where the first norm depends on the kernel $K(x)$ and the second norm is the vector norm of the momenta multiplied by a constant $\lambda^2$.  
In this case, we interpret the momenta as the conjugate variables to the template deformation vector field $\nu$ in order to have an equation only in term of the position and momenta of the landmarks. 
The derivation of the landmark equation is rather standard. Hence, we will only show it on the Hamiltonian side. 
We refer, for example, to \cite{holm2009euler} for more details of the deterministic derivation,  or to \cite{arnaudon2017geometric} and \cite{holm2016variational} for discussions of the stochastic landmark dynamics. 

Recall that the landmark Hamiltonian is
\begin{align}
  h_K(\mathbf p,\mathbf q) = \frac{1}{2}\sum_{ij} p_i\cdot  p_j K(q_i-q_j)\, , 
\end{align}
and the metamorphosis Hamiltonian is thus
\begin{align}
  h(q_i,p_i) = h_K(\mathbf q,\mathbf p) + \frac{\lambda^2}{2} \sum_{i=1}^n |p_i|^2\, . 
\end{align}
The stochastic potentials \eqref{sto-potential} become in this case 
\begin{align}
  \Phi_l^u(\mathbf q,\mathbf p) =  \sum_i p_i\cdot \sigma_l^u(q_i)\qquad \mathrm{and}\qquad  \Phi_i^\nu (\mathbf p) = p_i\cdot \sigma_i^\nu \, . 
\end{align}
Notice that the stochastic potential $\Phi^\nu$ is described by a fixed vector, where $\sigma^\nu_i$ is the amplitude of the noise for the landmark $i$. However, for the stochastic potential $\Phi^u$, we have to specify space (or $\mathbf q$) dependent functions $\sigma_l^u(\mathbf q)$. 
This simple form comes from the fact that we used a discrete set of points and $\nu= \mathbf p$ for the template deformation, and the summation over $k$ becomes a summation over the landmark index. In addition, a sum of two Wiener process is another Wiener process with the sum of the amplitude (if it is additive and in It\^o form). 
From this observation, one can see that the general equation $\Phi^\nu_k (\mathbf p) = \sum_i p_i \cdot \sigma_k^\nu$ is equivalent to a change of amplitudes $\sigma_k^\nu$ and $i=k$.  

We compute the stochastic Hamiltonian equations for landmarks to arrive at
\begin{align}
  \begin{split}
    {\rm d}q_i &= \frac{\partial h_K}{\partial p_i} {\rm d}t + \sum_l \sigma^u_l \circ {\rm d} W^l_t+ \lambda^2 p_i {\rm d}t + \sigma_i^\nu {\rm d}W^i_t \,, \\
    {\rm d}p_i &= - \frac{\partial h_K}{\partial q_i} {\rm d}t   + \sum_l \partial_{q_i} (p_i\cdot \sigma^u_l)\circ {\rm d}W^l_t\, , 
  \end{split}
\end{align}
in which we can use the It\^o integral for the $\nu$-noise, as it is additive. 

Notice that setting $\lambda=0$ recovers the standard landmark dynamics, but with an additive noise in the position equation. This is different from the conventional physical perspective, in which additive noise often appears in the momentum equation, as in \cite{TrVi2012,Vialard2013extension,marsland2017langevin}. 

\subsection{Images}

The present stochastic metamorphosis framework can be directly applied to images, by taking the template space $N$ to be the space of smooth functions from the domain $\Omega\subset \mathbb R^2$ to $\mathbb R$. 
We set $u_t\in \mathfrak X(\Omega)$ the deformation vector field and $\rho \in TN\cong N$ the template vector field. 
As before, the Lagrangian must have two parts, and the simplest non-trivial one is the sum of kinetic energies written as 
\begin{align}
  l(u,n,\nu) = \frac12 \|u_t\|_K^2+ \frac{\lambda^2}{2 } |\rho_t|_{L^2}^2\, , 
\end{align}
where the first norm depends on the kernel $K$ and the second norm is the standard $L^2$ norm over $\Omega$. 
By choosing a $L^2$ norm we can identify $\rho_t $ with its dual in the case $\lambda=1$.
We will thus not distinguish between $\sigma_t$ and $\nu_t$ of the general framework. 

Thus, as before, we use the Hamiltonian formulation of the stochastic metamorphosis equations with the stochastic potentials, 
\begin{align}
  \Phi^u_l(m_t) =\int_\Omega  \langle m_t(x), \sigma_l^u(x)\rangle dx\quad \mathrm{and} \quad \Phi^\nu_k(\sigma_t) = \int_\Omega \langle \rho_t(x), \sigma_l^\nu(x) \rangle dx\, . 
\end{align}
Notice that in this case, both $\sigma_l^u$ and $\sigma_l^\nu$ are functions of the domain $\Omega$, and they encode spatial correlation structure of the stochastic perturbations.

Then, because the Hamiltonian structure has three sorts of terms, the $\mathrm{ad}^*$, the $\diamond$ and the $\star$ terms defined in equation \eqref{ad-star-def}, which in this case are
\begin{align*}
  \mathrm{ad}^*_{u_t} m_t &= (u_t\cdot \nabla) m_t  + (m_t\cdot \nabla ) u_t + \mathrm{div}(u_t) m_t \,,\\
  \sigma_t\diamond \nu_t &= \sigma_t \cdot \nabla \nu_t \,,\\
  u_t\star \sigma_t & =\nabla\cdot (\sigma_t u_t)\, , 
\end{align*}
we arrive at the following set of stochastic PDEs (for any $\lambda$)
\begin{align}
  \begin{split}
    {\rm d}m_t &+ \mathrm{ad}^*_{u_t} m_t {\rm d}t + \sum_l \mathrm{ad}^*_{\sigma_l^u}m_t \circ {\rm d}W_t^l = \lambda^2 \rho_t\cdot \nabla \rho_t {\rm d}t + \sum_k \rho_t \cdot \nabla\sigma_k^\nu \circ {\rm d}W_t^k
  \,,\\
  {\rm d}\rho_t &+ \nabla\cdot ( \rho_t u_t) {\rm d}t + \sum_l \nabla\cdot(\rho \sigma_l^u ) \circ {\rm d}W_t^l= 0 \, . 
  \end{split}
\end{align}
Another important equation is the reconstruction relation \eqref{dn-reconstruction}, which now reads
\begin{align}
  {\rm d}g_t = u_t(g_t) {\rm d}t + \sum_l \sigma_l^u(g_t) \circ {\rm d}W_t^l + \rho_t {\rm d}t + \sum_k \sigma_k^\nu \circ {\rm d}W_t^k\, . 
\end{align}

Notice that if we set $\lambda=1$, the effect of the density, or template motion on the momentum $m$ only appears via the noise term, similarly to the landmark case. 

In the one dimensional case, the metamorphosis equation is known to reduce to the so-called CH2 system, which is equation coupling the Camassa-Holm equation with a density advection equation for $\rho_t= \nu_t$. 
We refer to \cite{holm2009euler,chen2006two} and references therein for more details about this equation and its complete integrability in the deterministic case.
A similar reduction holds for both stochastic deformations, and we have the following stochastic CH2 equation
\begin{align}
  \begin{split}
    {\rm d}m &+ (u\partial_x m + 2m\partial_x  u) {\rm d}t \\&= - \rho \partial_x \rho {\rm d}t 
    - \sum_k \rho \partial_x \sigma^\nu_{k} \circ {\rm d}W_t^k - \sum_l\left (  \sigma^u_l \partial_x m + \sum_l 2m \partial_x \sigma^u_{l}\right) \circ {\rm d}W_t^l
    \,,\\
    {\rm d}\rho &+ \partial_x (\rho u ) {\rm d}t +\partial_x (\rho \sigma_l^u)\circ {\rm d}W_t^l = 0 \, . 
  \end{split}
\end{align}

Compared to the landmark example, the noise associated to the template dynamics is described by a set of functions of the image, not a set of fixed vectors. 
The difference between the nature of these two types of noise is thus less apparent, apart from how they appear in the equation. 

\section{Perspectives}\label{perspective}
\subsection{Computational Anatomy}

Estimation of population atlases and
longitudinal analysis of anatomical changes caused by disease progression constitute integral parts of computational
anatomy \cite{younes_evolutions_2009}. The relation between these problems and the stochastic metamorphosis model presented here can be illustrated by the analysis of longitudinal brain MR-image data of patients suffering from Alzheimer's disease. The data manifold $N$ is here a vector space of images as described above with $\Omega\subseteq\mathbb R^3$.
\begin{figure}[!th]
\centering
  \vspace{.4cm}
\mbox{
\hspace{-.8cm}
\def\svgwidth{0.34\columnwidth}
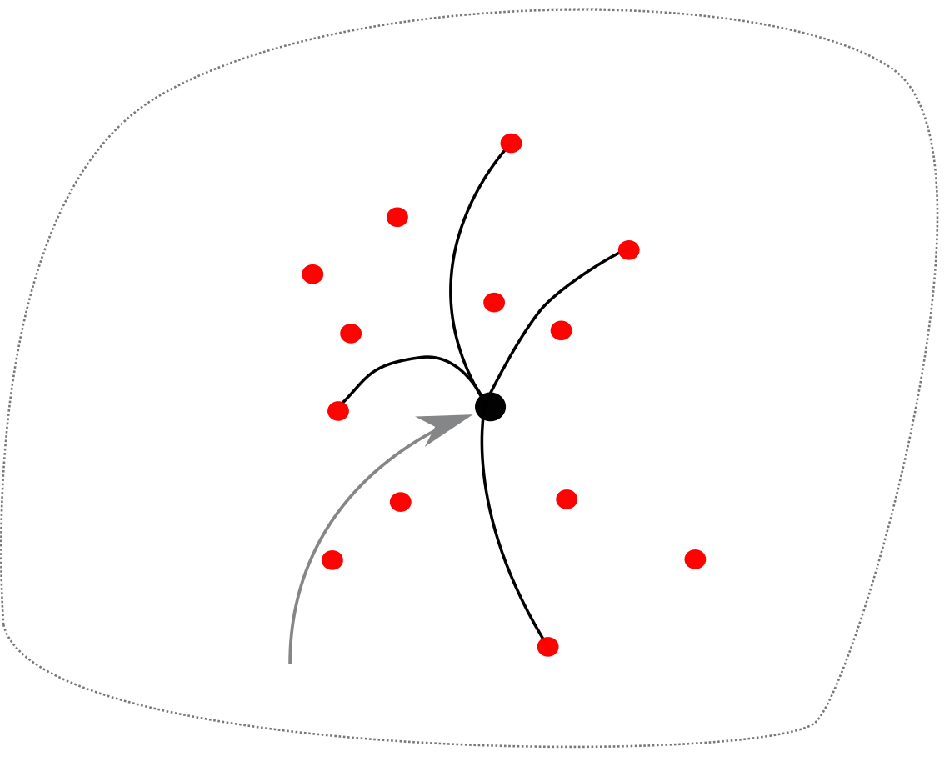
\hspace{-.0cm}
\def\svgwidth{0.43\columnwidth}
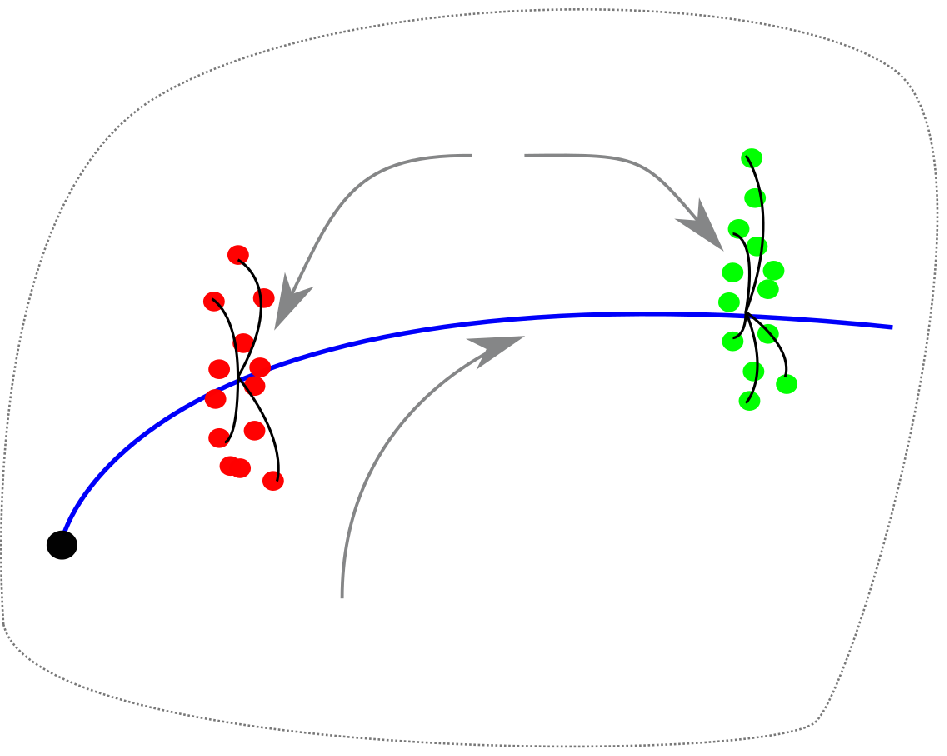
\hspace{-1.5cm}
\def\svgwidth{0.39\columnwidth}
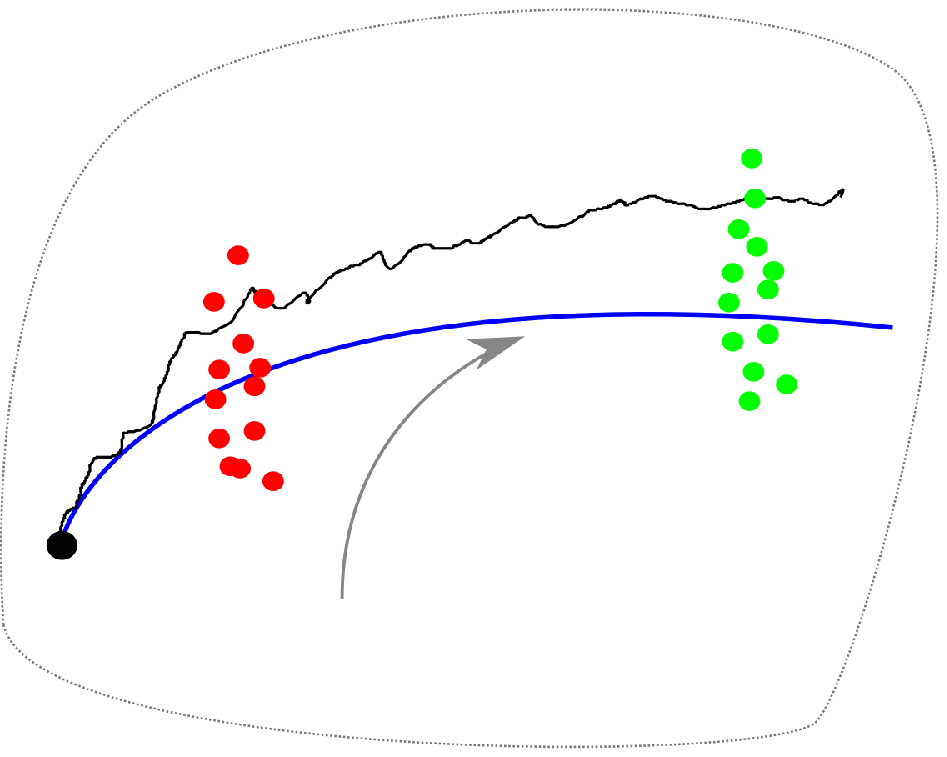
}
\caption{(left) Template estimation in the form \eqref{eq:LDDMMtemplate} aims at
finding a single descriptor $\eta$ for the population average of the observed
shapes $n^1,\ldots,n^k$ (red dots) in the non-linear shape space $N$. The variational 
principle \eqref{eq:LDDMMtemplate}
corresponds to assuming $n^i$ arise from geodesic perturbations of $\eta$.
(center) Geodesic regression models a population trend as a geodesic $n_t$.
Observations at different time points ($n_{t_1}^i$ red, $n_{t_2}^i$ green) arise
as perturbations of the points $n_{t_1}$ and $n_{t_2}$ by random geodesics. (right) 
Stochastic metamorphosis models the evolution of the population trend
$n_t$ deterministically while observations $n_{t_j}^i=n_{t_j}(\omega^i)$ 
appear from individual noise realizations $\omega^i$. The perturbations
are time continuous and apply to each case $i$ individually making the model
natural for modelling longitudinal evolutions with noise.
}
\label{fig:relations_comp_anatomy}
\end{figure}

Focusing first on template estimation, in medical imaging commonly denoted 
atlas estimation, the aim is to find a population average of data assumed observed
at a fixed time point. In the literature, this is for example 
pursued by minimizing the total sum of the regularized LDDMM energies 
of deterministic geodesic trajectories that deform the atlas to match the
observed data \cite{joshi_unbiased_2004}. For $k$ data points $n^1,\ldots,n^k$ and 
with data matching term $S:N\times N\rightarrow\mathbb R$, the template $\eta$
is then estimated by joint minimization of
\begin{equation}
  \min_{(\eta,v_t^1,\ldots,v_t^k)}
  \sum_{i=1}^k
  \int_0^1\|v_t^i\|^2{\rm d}t
  +
  S(\phi_T^i.\eta,n^i)
  \,,
  \label{eq:LDDMMtemplate}
\end{equation}
where the deformations $\phi_T^i$ each are endpoints of the integral of the vector fields $v_t^i$ on an interval $[0,T]$. 

A different approach to atlas estimation is to perform inference in 
statistical models
where observations are assumed random perturbations of a template
and inference of the template is performed via maximum-likelihood or
maximum-a-posteriori estimation. This approach is pursued, for example, in a
\cite{allassonniere_towards_2007,zhang_bayesian_2013,pai_statistical_2016}. See also the discussion below.

Longitudinal analysis aims at capturing the average time evolution of the brain shape caused by the disease
\cite{muralidharan_sasaki_2012,niethammer_geodesic_2011}. A common approach here is to estimate a general deterministic trend that is perturbed by noise at discretely observed time points in order to describe the observed images \cite{fletcher_probabilistic_2016}. 
For example, the noise can take the form of random initial velocity vectors for geodesics emanating from the deterministic trajectory.

The stochastic metamorphosis framework proposed here combines deterministic longitudinal evolution of
the template in both shape, represented by the deformations $g_t$, and in the
template image, $n_t=g_t.\eta_t$. We can assume longitudinal observations $n_{t_j}^i$, $i=1,\ldots,k$,
$j=1,\ldots,t_l$ at $l$ time points are realizations of the stochastic model with
time-continuous noise process drawn for each subject $i$.
The stochastic perturbations are thus tied to each subject affecting the
dynamics simultaneously with the evolution of the deterministic flow. The 
relation between this model, geodesic regression models, and atlas estimation is
illustrated in Figure~\ref{fig:relations_comp_anatomy}.

Because of the randomness, algorithms for inference of the template
$\eta$ and its evolution $n_t=g_t.\eta_t$ from data can naturally be formulated 
by matching statistics of the data, e.g. by matching moments or by
maximum-likelihood as done for the landmark case of stochastic EPDiff equations
in \cite{arnaudon2017geometric}. Development of such inference schemes constitutes
natural future research directions.

\subsection{Phase and Amplitude in Functional Data Analysis}
While images exhibit variations in both intensity and shape of the
image domain, signals in functional
data analysis often exhibit combined variation in amplitude and phase.
For a signal $f:I\rightarrow N$ defined on an interval $I$, amplitude variations refer to 
variations of the values $f(s)$ in $N$ for each fixed $s\in I$ while phase variation
covers changes in the parametrization of the domain $I$. This is illustrated
with $N=\mathbb R$ in Figure~\ref{fig:phase_amp}. An example of such combined phase and amplitude variations is provided in the growth curves of children and young adults; in which phase variation
is connected to the absolute height of the subject while phase variation
arise from growth and growth spurts occurring at different ages for different children.
\begin{figure}
    \centering
    \subfigure[tempate signal $\eta$]{\includegraphics[scale=0.28,trim=40 0 40 0,clip=true]{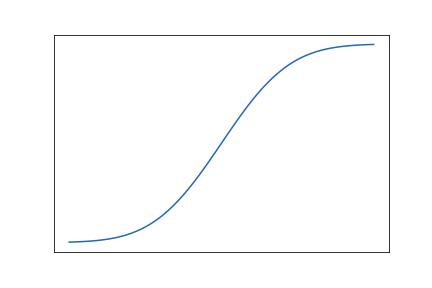}}
    \subfigure[phase variation ($\phi$)]{\includegraphics[scale=0.28,trim=40 0 40 0,clip=true]{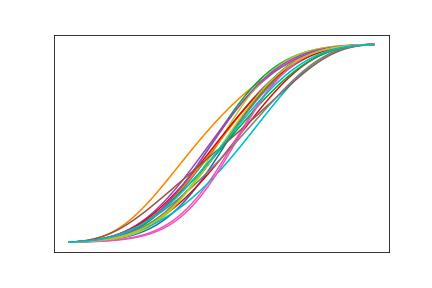}}
    \subfigure[amplitude variation ($\nu$)]{\includegraphics[scale=0.28,trim=40 0 40 0,clip=true]{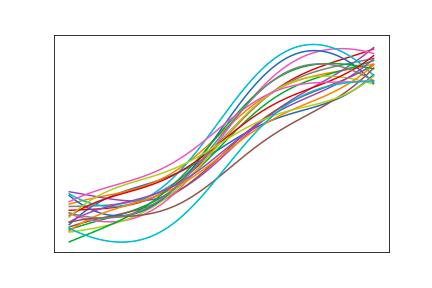}}
    \subfigure[phase and amplitude variation]{\includegraphics[scale=0.28,trim=40 0 40 0,clip=true]{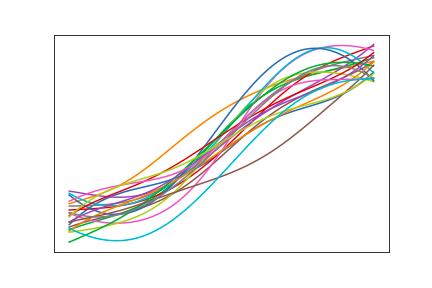}}
    \caption{A template signal (a) can be perturbed by 
      (b) variation in phase, in \eqref{eq:phase_amplitude} denoted $\phi$;
      (c) variation in amplitude, $\nu$ in \eqref{eq:phase_amplitude};
      (d) phase and amplitude simultaneously.
    }
    \label{fig:phase_amp}
\end{figure}

Recent literature
covers multiple approaches for identifying, separating and performing inference in
situations with combined phase and amplitude variation
\cite{raket2014nonlinear,marron_functional_2015,tucker_generative_2013}.
One example of a generative model in this settings is the mixed-effects model
\cite{raket2014nonlinear,kuhnel_most_2017}
\begin{equation}
  f(s)
  = 
  \eta(\phi^{-1}(s))
  +
  \nu(s)
  +
  \epsilon
  \ ,\ s\in I
  \,,
  \label{eq:phase_amplitude}
\end{equation}
where the average signal $\eta$ is deformed in phase by the action
$\phi.\eta=\eta\circ\phi^{-1}$ of a deformation $\phi$ of the interval $I$, and
in amplitude by the additive term $\nu$. Here $\eta$ is considered a fixed, non-random
effect while both $\phi$ and $\nu$ are random. Illustrated with the growth curve
case above, $\eta$ models the population average growth curve for each age $s$, while
$\phi$ controls the timing of the growth process for the individual children and $\nu$
the absolute height difference to the population average.
One observes that the model
\eqref{eq:phase_amplitude} is non-linear, because of the coupling between $\phi$
and $\eta$. In addition, a model for the deformations $\phi$ is needed, and the
randomness appearing in both $\phi$ and $\nu$ must be specified.

Whereas the LDDMM model is widely used in image analysis, this framework has not yet seen many applications for modelling deformations in functional data analysis, such as the phase variation appearing in \eqref{eq:phase_amplitude}.
Instead, works such as
\cite{raket2014nonlinear} use a small-deformation model $\phi(s)=s+v(s)$ with
random vector field $v$ modelling displacements on $I$. 
On the other hand, large-deformation flow models such as LDDMM
traditionally have not integrated random variation directly into the dynamics. 
Natural families of probability
distributions and generative models taking values in non-linear spaces such as
deformation spaces are
generally non-trivial to construct. However, the model proposed in this paper 
achieves exactly that.

A direct metamorphosis equivalent of the mixed-effects model \eqref{eq:phase_amplitude}
has $\eta=\eta_0$ the population average $\eta$, sets $u_0=\nu_0=0$ and encodes
the random effects $\phi$ and $\nu$ in \eqref{eq:phase_amplitude} in the
stochastic increments $du_t$ and $d\nu_t$. The action of $g_t$
on the signal is the right action
$g_t.f=f\circ g_t^{-1}$ as in \eqref{eq:phase_amplitude}. Now $d\nu_t$ models pure amplitude variation, $du_t$
phase variation, and the combined stochastic evolution of the signal is
$df_t=dn_t$.
We then assume the observed signal is $f=f_T$ for a fixed end time $T$ of the
stochastic process. Spatial correlation in both the deformation increments $du_t$ and the
amplitude increments $d\nu_t$ is encoded in the fields $\sigma_l^u$ and
$\sigma_k^\nu$ respectively.

In the above model, the template is stationary in time when disregarding the stochasticity. However, allowing
non-zero initial momenta $u_0$ and $\nu_0$ in both phase and amplitude allows
the template to vary with time and thereby gives 
a non-linear generalization of a standard multivariate regression model with one latent
variable for phase and one for amplitude. This in particular allows
modelling of trends over populations where subjects are
affected by both the population trend and individual stochastic perturbations.

\subsection{Statistical nonlinear modelling}
  It may initially seem overly complicated to use the metamorphosis framework for a simple regression model. However, statistical models that in linear space seem
  completely standard are often inherently difficult to generalize to non-linear
  spaces. In general, the lack of vector space structure makes distributions and 
  generative models hard to specify, see e.g. 
  \cite{sommer_anisotropic_2015,sommer_modelling_2017} for
  examples of the geometric complexities of generalizing the Euclidean normal distribution. 

  In Euclidean space, random vectors can model random perturbations. In non-linear spaces, the lack of vector space structure prevents this and random perturbations are often most naturally expressed with sequences of infinitesimal steps. Vectors are thus
  replaced with tangent bundle valued sequences that, when integrated over time, 
  give rise to stochastic flows. When modelling both deterministic and random variations, stochasticity generally couples non-trivially with the deterministic evolution. In addition, 
  perturbations and correlation structure must be specified with respect to a
  frame of reference. While Euclidean space provides a global coordinate system allowing this, a model of transport must be specified in non-linear spaces.
  The stochastic metamorphosis model is an example of a model coupling
  deterministic and stochastic evolution and using right-invariance to provide
  reference frames for the perturbations and correlation structure.
  An example of a related but different approach is
  \cite{kuhnel_stochastic_2017} where parallel transport is used to link covariance between tangent spaces.

\subsection*{Acknowledgements}
{\small AA acknowledges funding from the EPSRC through award EP/N014529/1 funding the EPSRC Centre for Mathematics of Precision Healthcare.
AA and DH are partially supported by the European Research Council Advanced Grant 267382 FCCA held by DH. DH is also grateful for partial support from EPSRC Grant EP/N023781/1.
SS is partially supported by the CSGB Centre for Stochastic Geometry and Advanced
Bioimaging funded by a grant from the Villum foundation. 
The authors would like to thank the Isaac Newton Institute for Mathematical Sciences for their support and hospitality during the programme {\it Growth form and self-organisation} when this paper was finished. This work is supported by EPSRC Grant Number EP/K032208/1. 

}

\bibliography{biblio.bib}

\end{document}